\documentclass{article}
\usepackage{geometry}
\usepackage{amsmath,amsthm}
\usepackage{amsfonts}
\usepackage{amssymb}
\usepackage{url}
\usepackage{xspace}
\usepackage{graphicx, epstopdf}
\usepackage{pdfsync}
\usepackage{pdflscape}
\usepackage{mathrsfs}
\usepackage{framed}
\usepackage[x11names,usenames,dvipsnames,svgnames,table]{xcolor}
\usepackage{listings}
\usepackage[inline]{enumitem}
\usepackage[hidelinks]{hyperref}

\newtheorem{theorem}{Theorem}
\newtheorem{lemma}{Lemma}
\newtheorem{proposition}{Proposition}
\newtheorem{corollary}{Corollary}

\theoremstyle{definition}

\newtheorem{remark}{Remark}

\newcommand{\bx}{\mathbf{x}}
\newcommand{\by}{\mathbf{y}}
\newcommand{\bz}{\mathbf{z}}
\newcommand{\be}{\mathbf{e}}
\newcommand{\bxout}{{\mathbf{x}_{\mathrm{out}}}}
\newcommand{\xout}{{x_{\mathrm{out}}}}
\newcommand{\bg}{\mathbf{g}}
\renewcommand{\be}{\mathbf{e}}
\newcommand{\bu}{\mathbf{u}}
\newcommand{\bn}{\mathbf{n}}
\newcommand{\bxi}{\boldsymbol{\xi}}
\newcommand{\Ncal}{\mathcal{N}}
\newcommand{\balpha}{\boldsymbol{\alpha}}
\newcommand{\bgamma}{\boldsymbol{\gamma}}
\newcommand{\WTC}{W_\mathcal{T}^C}

\newcommand{\real}{\mathbb{R}}

\newcommand{\reals}{\mathbb{R}}
\newcommand{\E}{\mathbb{E}}
\newcommand{\Var}{\mathbb{V}}

\newcommand{\Ocal}{\mathcal{O}}

\newcommand{\norm}[1]{\|#1\|}

\newcommand{\secref}[1]{Sec.~\ref{#1}}
\newcommand{\subsecref}[1]{Subsection~\ref{#1}}
\newcommand{\figref}[1]{Fig.~\ref{#1}}
\newcommand{\lemref}[1]{Lemma~\ref{#1}}

\newcommand{\thmref}[1]{Thm.~\ref{#1}}
\newcommand{\propref}[1]{Proposition~\ref{#1}}
\newcommand{\appref}[1]{Appendix~\ref{#1}}

\title{The Complexity of Finding Stationary Points\\ with Stochastic Gradient Descent}
\author{Yoel Drori\\Google Research \and Ohad Shamir\\Weizmann Institute of Science\\ and Google Research}

\begin{document}
\maketitle

\begin{abstract}

We study the iteration complexity of stochastic gradient descent (SGD) for minimizing the gradient norm of smooth, possibly nonconvex functions. We provide several results, implying that the $\mathcal{O}(\epsilon^{-4})$ upper bound of Ghadimi and Lan~\cite{ghadimi2013stochastic} (for making the average gradient norm less than $\epsilon$) cannot be improved upon, unless a combination of additional assumptions is made. Notably, this holds even if we limit ourselves to convex quadratic functions. We also show that for nonconvex functions, the feasibility of minimizing gradients with SGD is surprisingly sensitive to the choice of optimality criteria.

\end{abstract}

\section{Introduction}

Stochastic gradient descent (SGD) is today one of the main workhorses for solving large-scale supervised learning and optimization problems. Much of its popularity is due to its extreme simplicity: Given a function $f$ and an initialization point $\bx$, we perform iterations of the form 
$\bx_{t+1}=\bx_t-\eta_t \bg_t$, where $\eta_t>0$ is a step-size parameter and $\bg_t$ is a stochastic vector which satisfies $\E[\bg_t|\bx_t]=\nabla f(\bx_t)$. For example, in the context of machine learning, $f(\bx)$ might be the expected loss of some predictor parameterized by $\bx$ (over some underlying data distribution) and $\bg_t$ is the gradient of the loss w.r.t. a single data sample. For convex problems, the convergence rate of SGD to a global minimum of $f$ has been very well studied (for example, \cite{kushner2003stochastic,nemirovski2009robust,moulines2011non,bertsekas2011incremental,rakhlin2012making,bottou2018optimization}), however, for nonconvex problems, convergence to a global minimum cannot in general be guaranteed. A reasonable substitute is to study the convergence to local minima, or at the very least, to stationary points. This can also be quantified as an optimization problem where the goal is not to minimize $f(\bx)$ over $\bx$, but rather $\norm{\nabla f(\bx)}$. This question of finding stationary points has gained more attention in recent years, with the rise of deep learning and other large-scale nonconvex optimization methods.

Compared to optimizing function values, the convergence of SGD in terms of minimizing the gradient norm is relatively less well-understood. 
A result by Ghadimi and Lan~\cite{ghadimi2013stochastic}, which we repeat in \appref{sec:upperbounds} for completeness,
states that for smooth (Lipschitz gradient) functions, $\Ocal(\epsilon^{-4})$ iterations are sufficient to make the average expected gradient $\E[\frac{1}{T}\sum_{t=1}^{T}\norm{\nabla f(\bx_t)}]$ 
less than $\epsilon$, and it was widely conjectured that this is the best complexity achievable with SGD. However, this bound was recently improved in Fang et al.\ \cite{fang2019sharp}, which showed a complexity bound of $\Ocal(\epsilon^{-3.5})$ for SGD, under the following additional assumptions/algorithmic modifications:
\begin{enumerate}
    \item \label{as:aggregation} \textbf{(Complex) aggregation.} Rather than considering the average or minimal gradient norm of the iterates, the algorithm considers the norm of a certain adaptive average of a suffix of the iterates (those which do not deviate too much from the final iterate). 
    \item \label{as:hessian} \textbf{Lipschitz Hessian.} The function is twice differentiable, with a Lipschitz Hessian as well as a Lipschitz gradient.
    \item \label{as:noise} \textbf{``Dispersive'' noise.} The stochastic noise satisfies a ``dispersive'' property, which intuitively implies that it is well-spread (it is satisfied, for example, for Gaussian or uniform noise in some ball). 
    \item \label{as:dimension} \textbf{Bounded dimension.} The dimension is bounded, in the sense that there is an explicit logarithmic dependence on it in the iteration complexity bound (in contrast, the 
    Ghadimi and Lan $\Ocal(\epsilon^{-4})$ result is dimension-free).
\end{enumerate}

The result of Fang et al.\ is even stronger, as it shows convergence to a \emph{second-order} stationary point (where the Hessian is nearly positive definite), however, this will not be our focus here. Note that in this setting it is known that some dimension dependence is difficult to avoid (see~\cite{simchowitz2017gap})



In this paper, we study the performance limits of SGD for minimizing gradients through several variants of lower bounds under different assumptions. In particular, we wish to understand which of the assumptions/modifications above are necessary to break the $\epsilon^{-4}$ barrier. 
Our main take-home message is that most of these appear to be needed in order to attain an iteration complexity better than $\Ocal(\epsilon^{-4})$, in some cases even if we limit ourselves just to convex quadratic functions. In a bit more detail: \begin{itemize}
    \item If we drop Assumption~\ref{as:dimension} (bounded dimension), and consider the norm of the gradient at the output of some fixed, deterministic aggregation scheme (as opposed to returning, for example, an iterate with a minimal gradient norm), then perhaps surprisingly, we show that it is impossible to provide \emph{any} finite complexity bound. This holds under mild algorithmic conditions, which extend far beyond SGD. This implies that for dimension-free bounds, we must either consider rather complicated aggregation schemes, apply randomization, or use optimality criteria which do not depend on a single point (e.g., consider the average gradient  $\frac{1}{T}\sum_{t=1}^{T}\norm{\nabla f(\bx_t)}$ or $\min_t \norm{\nabla f(x_t)}$, as is often done in the literature). This result is formalized as \thmref{thm:infdim} in \subsecref{subsec:fixedpoint}.
    \item Without Assumption~\ref{as:hessian} (Lipschitz Hessian) and Assumption~\ref{as:noise} (dispersive noise), then even with rather arbitrary aggregation schemes, the iteration complexity of SGD is $\Omega(\epsilon^{-4})$. This result is formalized as \thmref{thm:aggregation_step} in \subsecref{subsec:sgdlowbound}.
    \item Without Assumption~\ref{as:aggregation} (aggregation) and Assumption~\ref{as:noise} (dispersive noise), the iteration complexity of SGD required to satisfy $\E[\min_t \norm{\nabla f(\bx_t)}]\leq \epsilon$ is $\Omega(\epsilon^{-3})$. This result is formalized as \thmref{thm:nonconvex} in \subsecref{subsec:sgdlowbound}.
    \item Without aggregation, the iteration complexity of SGD with ``reasonable'' step sizes to attain $\E[\min_t \norm{\nabla f(\bx_t)}]\leq \epsilon$ is $\Omega(\epsilon^{-4})$, even for quadratic \emph{convex} functions in moderate dimension and Gaussian noise (namely, all other assumptions are satisfied as well as convexity). This result is formalized as \thmref{thm:sgdlow} in Section \ref{sec:convex}. 
\end{itemize}


It is important to note that 
the SGD algorithm, which is the main focus of this paper, is not necessarily an optimal algorithm (in terms of iteration complexity) for minimizing gradient norms in our stochastic optimization setting.
For example, for convex problems, it is known that it is possible to achieve an iteration complexity of $\tilde{O}(\epsilon^{-2})$, which strictly smaller than our $\Omega(\epsilon^{-4})$ lower bound (see \cite{foster2019complexity}, and for a related result in the deterministic setting see~\cite{nesterov2012make}).
These algorithms are more complicated and less natural than plain SGD, a price that our results indicate might be necessary in order to achieve optimal iteration complexity in some cases.

We conclude this section by noting that following the initial dissemination of our paper, a recent arXiv preprint \cite{arjevani2019lower} studied a similar question of lower complexity bounds for finding stationary points, focusing on algorithm-independent $\Omega(\epsilon^{-4})$ or $\Omega(\epsilon^{-3})$ lower bounds for functions with Lipschitz-continuous gradients. Their results are mostly incomparable to ours. In particular, our \thmref{thm:infdim} studies conditions under which no finite lower bound is possible, \thmref{thm:nonconvex} considers the case where the Hessian (and not just the gradient) is Lipschitz-continuous, and \thmref{thm:sgdlow} shows an $\Omega(\epsilon^{-4})$ lower bound for SGD, which holds \emph{even} if the functions are convex and the noise is simply Gaussian (in contrast, the constructions in \cite{arjevani2019lower} crucially depend on intricate non-convex functions and carefully tailored,  location-dependent noise, using a considerably more involved proof). The result most similar to those in \cite{arjevani2019lower} is \thmref{thm:aggregation_step}, which is specific to SGD, but admits a simpler proof and significantly better constants.


\section{Setting and Notation}\label{sec:setting}

We let bold-face letters denote vectors, use $\be_i$ to denote the canonical unit vector, and use $[T]$ as shorthand for $\{1,2,\ldots,T\}$.

We assume throughout that the objective $f$ maps $\reals^d$ to $\reals$, and
either has an $L$-Lipschitz gradient for some fixed parameter $L>0$ or a $\rho$-Lipschitz Hessian for some $\rho>0$.

We consider algorithms which use a standard stochastic first-order oracle \cite{Book:NemirovskyYudin,agarwal2009information} in order to minimize some optimality criteria: This oracle, given a point $\bx_t$, returns $\nabla f(\bx_t)+\bxi_t$, where $\bxi_t$ is a random variable satisfying 
\[
\E[\bxi_t|\bx_t]=0~~~\text{and}~~~\E[\norm{\bxi_t}^2|\bx_t]\leq \sigma^2
\]
almost surely for some fixed $\sigma^2$. 
In this paper, we focus on optimality criteria involving minimizing gradient norms, using the Stochastic Gradient Descent (SGD) algorithm.
This algorithm, given a budget of $T$ iterations and an initialization point $\bx_1$, produces $T$ stochastic iterates $\bx_1,\ldots,\bx_T$ according to 
\begin{equation}\label{eq:sgd}
 \bx_{t+1}~=~\bx_t-\eta_t \cdot (\nabla f(\bx_t)+\bxi_t)~,    
\end{equation}
where $\eta_t$ is a fixed step-size parameter. In some cases, we will also allow the algorithm to perform an additional aggregation step, generating a point $\bxout$ which is some function of $\bx_1,\ldots,\bx_T$ (for example, the average $\frac{1}{T}\sum_{t=1}^{T}\bx_t$). Additionally, in some of our results, we will allow the step size to be adaptive, and depend on the previous iterates (under appropriate assumptions), in which case we will use the notation
\begin{equation}\label{eq:adaptive_sgd}
 \bx_{t+1} =\bx_t - \eta_{\bx_1,\dots,\bx_t} \cdot (\nabla f(\bx_t)+\bxi_t). 
\end{equation}
Regarding the initial conditions, we make the standard assumption\footnote{See e.g.~\cite{Book:Nesterov} and references mentioned earlier.} that $\bx_1$ has bounded suboptimality, i.e.,
\[
    f(\bx_1)-f(\bx_*)\leq \Delta
\]
for some fixed $\Delta > 0$, where in the convex case, we assume $\bx_*$ is some point $\bx_* \in \arg\min_{\bx}f(\bx)$, and in the non-convex case, we assume $\bx_*$ is a stationary point with $f(\bx_*) \leq f(\bx_t)$ for all $t\in [T]$. We note that some analyses (see for example \cite{allen2018make,foster2019complexity}) replace the assumption $f(\bx_1)-f(\bx_*)\leq \Delta$ with the assumption $\norm{\bx_1-\bx_*}\leq R$, but we do not consider this variant in this paper (in fact, some of our constructions rely on the fact that even if $f(\bx_1)-f(\bx_*)$ is small, $\norm{\bx_1-\bx_*}$ might be very large). It should also be pointed out that in the non-convex setting, $\bx_*$ might not be uniquely defined or even belong to a single connected set, which makes $\norm{\bx_1-\bx_*}$ somewhat ambiguous.




\section{Lower bounds in the non-convex case}\label{sec:nonconvex}
In this section, we present several lower bounds relating to first-order methods in the non-convex stochastic setting. We start by considering a wide range of first-order methods, showing that if we consider any point which is a \emph{fixed} function of the iterates, then \emph{no} meaningful, dimension-free worst-case bound can be attained on its expected gradient norm. We conclude that it is necessary for any useful optimality criterion to relate to more than one iterate in some way, 
as is indeed the case with the standard optimality criteria, which considers the average expected norm of the gradients ($\frac{1}{T} \sum_t \E\|\nabla f(\bx_t)\|$) or the minimal expected norm of the gradients ($\min_t \E\|\nabla f(\bx_t)\|$).

We then turn our focus to the SGD method under the standard set of assumptions (see \secref{sec:setting}), and show that it requires $\Omega(\epsilon^{-4})$ iterations (or $\Omega(\epsilon^{-3})$ with Lipschitz Hessians) to attain a value of $\epsilon$ for any of the standard optimality criteria mentioned above.

\subsection{Impossibility of minimizing the gradient at any fixed point}
\label{subsec:fixedpoint}

In this subsection, we show that in the nonconvex setting, perhaps surprisingly, \emph{no} meaningful iteration complexity bound can be provided on $\norm{\nabla f(\bxout)}$, where $\bxout$ is the point returned by any fixed, deterministic aggregation scheme which depends continuously on the iterates and stochastic gradients (for example, some fixed weighted combination of the iterates). 

To state the result, recall that SGD can be phrased in an oracle-based setting, where we model an optimization algorithm as interacting with a stochastic first-order oracle: Given an initial point $\bx_1$, at every iteration $t=2,\ldots,T$, the algorithm chooses a point $\bx_t$, and the oracle returns a stochastic gradient estimate $\bg_t:=\nabla f(\bx_t)+\bxi_t$, where $\E[\bxi_t|\bx_t]=0$ and $\E[\norm{\bxi_t}^2|\bx_t]\leq \sigma^2$ for some known $\sigma^2$. The algorithm then uses $\bg_t$ (as well as $\bg_1,\ldots,\bg_{t-1}$ and $\bx_1,\ldots,\bx_{t}$) to select a new point $\bx_{t+1}$. After $T$ iterations, the algorithm returns a final point $\bxout$, which depends on $\bg_1,\ldots,\bg_T$ and $\bx_1,\ldots,\bx_T$. 

\begin{theorem}\label{thm:infdim}
Consider any deterministic algorithm as above, which satisfies the following:
\begin{itemize}
    \item There exists a finite $C_{T}$ (dependent only on $T$) such that for any initialization $\bx_1$ and any $t\in [T]$, if $\bg_1=\ldots=\bg_{t}=\mathbf{0}$, then $\norm{\bx_{t+1}-\bx_1}\leq C_{T}$. Moreover, if this holds for $t=T$, then $\norm{\bxout-\bx_1}\leq C_T$.
    \item For any $t\in [T]$, $\bx_{t+1}$ is a fixed continuous function of $\bx_1,\bg_1,\ldots,\bx_t,\bg_t$, and $\bxout$ is a fixed continuous function of $\bx_1,\bg_1,\ldots,\bx_T,\bg_T$.
\end{itemize}
Then for any $\delta\in (0,1)$, and any choice of random variables $\bxi_t$ satisfying the assumptions above, there exists a dimension $d$, a twice-differentiable function $f:\reals^d\mapsto\reals$ with $2$-Lipschitz gradients and $4$-Lipschitz Hessians, and an initialization point $\bx_1$ satisfying $f(\bx_1)-\inf_\bx f(\bx)\leq 1$, such that $\norm{\nabla f(\bxout)}\geq \frac{1}{2}$ holds with probability at least $1-\delta$.	Moreover, if there is no stochastic noise ($\bxi_t\equiv \mathbf{0}$), then the result holds for $d=1$.
\end{theorem}

Intuitively, the first condition in the theorem requires that the algorithm does not ``move'' too much from the initialization point $\bx_1$ if all stochastic gradients are zero (this is trivially satisfied for SGD, and any other reasonable algorithm we are aware of), while the second condition requires the iterates produced by the algorithm to depend continuously on the previous iterates and stochastic gradients (again, this is satisfied by SGD). 
By constructing a one-dimensional function whose gradient is zero over two disjoint regions (see \figref{fig:sfunc}), these two conditions allow the application of the intermediate value theorem to find an initialization point such that function value at $\bxout$ attains any value in between the function values at the regions.

The theorem suggests that to get non-trivial results, we must either use a dimension-dependent analysis, use a non-continuous/adaptive/randomized scheme to compute $\bxout$, or measure the performance of the generated sequence using an optimality criterion that does not depend on a fixed point (e.g., the average gradient $\frac{1}{T}\sum_{t=1}^{T}\norm{\nabla f(\bx_t)}$ or $\min_{t}\norm{\nabla f(\bx_t)}$). We note that the positive result of \cite{fang2019sharp} assumes both finite dimension, and computes $\bxout$ according to an adaptive non-continuous decision rule (involving branching depending on how far the iterates have moved), hence there is no contradiction to the alluded theorem.

\begin{proof}[Proof of \thmref{thm:infdim}]
	We will first prove the result in the case where there is no noise, i.e. $\bg_t=\nabla f(\bx_t)$ deterministically, in which case $\bx_2,\ldots,\bx_{T}$ and $\bxout$ are deterministic functions of $\bx_1$. 
	To that end, let $d=1$ and let $f(x)=s(x)$, where $s$ is the sigmoid-like function (see \figref{fig:sfunc})
	\begin{align*}
	s(x)=\begin{cases}
	-\frac{1}{2} & x\leq -1,\\
    \frac{2}{3}(x+1)^3 - \frac{1}{2} & x\in [-1,-\frac{1}{2}], \\
	-\frac{2}{3}x^3+x & x\in [-\frac{1}{2}, \frac{1}{2}], \\
	\frac{2}{3}(x-1)^3+\frac{1}{2} & x\in [\frac{1}{2},1], \\
	\frac{1}{2} & x\geq 1.\end{cases}
	\end{align*}
	This function smoothly and monotonically interpolates between $-1/2$ at $x=-1$ at $1/2$ at $x=1$. It can be easily verified to have $2$-Lipschitz gradients and $4$-Lipschitz Hessians, and for any $x$, satisfies 
	$f(x)-\inf_{x}f(x) \leq 1$. 

    \begin{figure}
    \vskip 0.2in
    \begin{center}
    \centerline{\includegraphics[width=0.5\columnwidth]{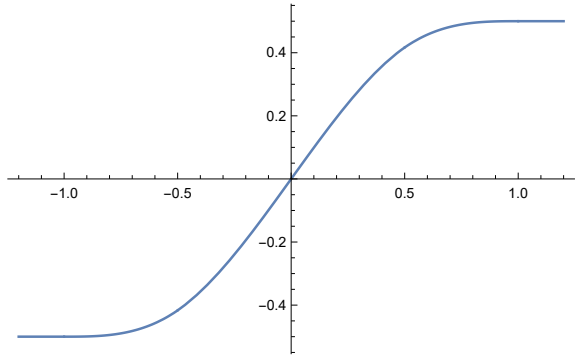}}
    \caption{The function $s(x)$.}
    \label{fig:sfunc}
    \end{center}
    \vskip -0.2in
    \end{figure}

	Let us consider the
	iterates generated by the algorithm, $x_1,\dots,x_T$ and $\xout$, as we make $x_1\rightarrow\infty$. Our function is such that $\nabla f(x)=0$ for all $x\geq 1$, so at every iteration, the algorithm gets $g_t=0$ as long as $x_t \geq 1$. Moreover, by the assumptions, as long as the gradients are zero, $|x_t-x_1|$ is bounded. As a result, by induction and our assumption that $|\xout-x_1|$ is bounded, we get that $\xout\rightarrow\infty$. A similar argument shows that when $x_1\rightarrow -\infty$, we also have $\xout\rightarrow -\infty$. 
	
	Next, we argue that $\xout$ is a continuous function of $x_1$. Indeed, $x_2$ is a continuous function of $x_1$, since it is a continuous function of $g_1=\nabla f(x_1)$ by assumption, and $\nabla f(x_1)$ is Lipschitz (hence continuous) in $x_1$, and compositions of continuous functions is continuous. By induction, a similar argument holds for $x_t$ for any $t$, and hence also to $\xout$.
	
	Overall, we showed that $\xout$ is a continuous function of $x_1$, that $\xout\rightarrow\infty$ when $x_1\rightarrow\infty$, and that $\xout\rightarrow-\infty$ when $x_1\rightarrow-\infty$.  Therefore, by the mean value theorem, there exists 
	some $x_1$ for which $\xout$ is precisely zero, in which case $|f'(\xout)|=|f'(0)|=|s'(0)|=1$, satisfying the Theorem statement. 
	
	It remains to prove the theorem in the noisy case, where $\bxi_i$ are non-zero random variables. In that case, instead of choosing $f(x)=s(x)$, we let $f(\bx)=s(\langle \bx, \be_r \rangle)$, where the coordinate $r$ is defined as 
	\[
	r:=\arg\min_{j\in [d]} \max_{t\in [T]}\E[\langle\bxi_t, \be_j\rangle^2].
	\]
	Since $\max_t \E[\norm{\bxi_t}^2]=\max_t \sum_{j=1}^{d}\E[\langle\bxi_{t}, \be_j\rangle^2]$ is bounded by $\sigma^2$ independently of $d$, it follows that the variance of $\bxi_1,\ldots,\bxi_t$ along coordinate $r$ goes to zero as $d\rightarrow \infty$. Therefore, by making $d$ large enough and using Chebyshev's inequality, we can ensure that $\max_t |\langle \bxi_{t},\be_j\rangle|$ is arbitrarily small with arbitrarily high probability. Since the gradients of $f$ are Lipschitz, and we assume each $\bx_{t+1}$ is a continuous function of the noisy gradients $\bg_1,\bg_2,\ldots,\bg_t$, it follows that the trajectory of $\bx_1,\bx_2,\ldots,\bx_{T}$ and $\bxout$ on the $j$-th coordinate can be made arbitrarily close to the noiseless case analyzed earlier (where $\bxi_t\equiv \mathbf{0}$), with arbitrarily high probability. In particular, we can find an initialization point $\bx_1$ such that the $j$-th coordinate of $\bxout$ is arbitrarily close to $0$, hence the gradient is arbitrarily close to $1$ (and in particular, larger than $1/2$). 
\end{proof}

\begin{remark}[Randomized Algorithms]
The theorem considers deterministic algorithms for simplicity, but the same proof idea holds for larger families of randomized algorithms, where the randomness is used ``obliviously''. For example, consider the popular technique of adding random perturbations to the iterates: If the perturbations have a fixed distribution with finite variance, then we can always embed our construction in a high enough dimension, so that the effective variance of the perturbations is arbitrarily small, and we are back to the deterministic setting.
\end{remark}

\subsection{Lower bounds on SGD}\label{subsec:sgdlowbound}

In this subsection, we focus on the analysis of SGD in the nonconvex setting. We present two main results: A lower bound on the performance of SGD with an aggregation step for objectives with $L$-Lipschitz gradient, followed by a lower bound in the case where the objective has $\rho$-Lipschitz Hessian that applies to ``plain'' SGD methods that do not perform an aggregation step.
In both cases, the step sizes chosen by the method are allowed to be adaptive, in the sense that they are allowed to depend on past iterates and gradients. This dependence is \emph{not} allowed to be completely general, but rather we assume that the dependence on the past iterates and gradients is done through a function of their norm and the dot-products between them (in the Lipschitz Hessian case, we also allow the step size to depend on the Hessians).
Note, that all commonly used adaptive schemes (including Adagrad~\cite{duchi2011adaptive}, normalized gradient~\cite{nesterov1984minimization,kiwiel2001convergence}, among others) follow this type of adaptive scheme.

We start the analysis with a technical lemma. 

\begin{lemma}\label{L:detbound}
Let $f:\reals^d \mapsto \reals$ be a function with $L$-Lipschitz gradient, and assume that the vectors $\by_1,\dots,\by_m, \bz_1,\dots,\bz_n, \bgamma\in \reals^d$ ($n,m\in\mathbb{N}$) are such that 
\begin{enumerate}
    \item\label{itm:grad_eq} $\nabla f(\by_1) = \dots = \nabla f(\by_m)= \nabla f(\bz_1) = \dots = \nabla f(\bz_n) = \bgamma$,
    \item\label{itm:constant_vals} $f(\by_1)=\dots=f(\by_m)$, and
    \item\label{itm:constant_prod} $\langle \bgamma, \by_1\rangle = \dots = \langle \bgamma, \by_m\rangle$.
\end{enumerate}
Then there exists a function $\hat f$ with $L$-Lipschitz gradient that has the same first-order information as $f$ at $\{\bz_i\}_{i\in [n]}$, i.e., for all $i\in [n]$
\begin{align*}
    & \hat f(\bz_{i}) = f(\bz_{i}), \\
    & \nabla \hat f(\bz_i) = \nabla f(\bz_i) = \bgamma,
\end{align*}
has the same gradient as $f$ at $\{\by_i\}_{i\in [n]}$, i.e., for all $j\in [m]$
\begin{align*}
    & \nabla \hat f(\by_j) = \nabla f(\by_j) = \bgamma,
\end{align*}
and is bounded from below:
\[
    \inf_{x \in\reals^d} \hat f(x) \geq \min_{k\in [n]} f(\bz_k) - \frac{3}{2L} \|\bgamma\|^2.
\]
\end{lemma}

We postpone the proof of this lemma to the appendix and turn to present the first main result of this subsection.

\begin{theorem}\label{thm:aggregation_step}
Consider a first-order method that given a function $f:\reals^d\rightarrow\reals$ and an initial point $\bx_1\in \reals^d$ generates a sequence of $T\in\mathbb{N}$ points $\{\bx_i\}$ satisfying
\begin{align*}
     & \bx_{t+1} = \bx_t + \eta_{\bx_1,\dots,\bx_t} \cdot (\nabla f(\bx_t)+\bxi_t), \quad t\in [T-1],
\end{align*}
where $\bxi_i$ are some random noise vectors, and returns a point $\bxout \in \reals^d$ as a non-negative linear combination of the iterates:
\[
    \bxout = \sum_{t=1}^T \zeta^{(t)}_{\bx_1,\dots,\bx_T} \bx_t.
\]
We further assume that the step sizes $\eta_{\bx_1,\dots,\bx_t}$ and aggregation coefficients $\zeta^{(t)}_{\bx_1,\dots,\bx_T}$ are deterministic functions of the norms and inner products between the vectors $\bx_1,\dots,\bx_t, \nabla f(\bx_1)+\bxi_1,\dots, \nabla f(\bx_t)+\bxi_t$.
Then for any $L,\Delta, \sigma\in \reals_{++}$ there exists a function $f:\reals^{T}\mapsto \reals$ with $L$-Lipschitz gradient, a point $\bx_1\in \reals^T$ and independent random variables $\bxi_t$ with $\E[\bxi_t]=0$ and $\E[\|\bxi_t\|^2]=\sigma^2$ such that
\begin{align*}
    & f(\bx_1) - \inf_\bx f(\bx) \overset{\text{a.s.}}{\leq} \Delta, \\
    & \nabla f(\bx_t) \overset{\text{a.s.}}{=} \bgamma, \qquad \forall t\in[T],\\
    & \nabla f(\bxout) \overset{\text{a.s.}}{=} \bgamma,
\end{align*}
where $\bgamma \in \reals^T$ is a vector such that
\begin{align*}
    \|\bgamma\|^2 &= \frac{\sigma}{16(T-1)} \left(\sqrt{64 L \Delta (T-1)+9 \sigma ^2}-3 \sigma \right) \\
    & \underset{T\gg 1}{\approx} \frac{\sigma}{2} \sqrt{\frac{L \Delta}{T-1}}.
\end{align*}
\end{theorem}

\begin{proof}
We will assume that the algorithm performs gradient steps with fixed step-size $\eta_t$ and aggregation coefficients $\zeta_i$, i.e., the algorithm is defined by the rule
\begin{align*}
    & \bx_{t+1} = \bx_t - \eta_t (\nabla f(\bx_t) + \bxi_t), \quad t\in[T-1],  \\
    & \bxout = \sum_{t=1}^T \zeta_i \bx_i.
\end{align*}
The analysis of the general case appears in the appendix.

Under this assumption, the proof proceeds by 
\begin{enumerate*}
\item
defining an adversarial objective and noise distribution, 
\item showing that the gradients of the objective at the iterates posses the claimed properties, then
\item using \lemref{L:detbound}, modifying the objective so that the claimed lower bound on the function is attained, while keeping the behavior of the function at the iterates unaffected.
\end{enumerate*}

We start by defining an adversarial example, choosing the noise vectors $\{\bxi_t\}$ to be independent random variables distributed such that
\begin{equation*}\label{def:nonconvex_xi}
    P(\bxi_t=\pm \sigma \be_{t+1}) = \frac{1}{2}, \quad t\in [T-1],
\end{equation*}
where $\be_i$ stands for the canonical unit vector, and defining the objective $f:=f_{\{\eta_t\},\{\zeta_t\}}:\real^{T} \mapsto \real$ by
\begin{equation*}\label{def:nonconvex_f}
    f_{\{\eta_t\},\{\zeta_t\}}(\bx) := G \cdot \langle \bx, \be_1\rangle + \sum_{t=1}^{T-1} h_t (\langle \bx, \be_{t+1}\rangle),
\end{equation*}
where $G \geq 0$ is a number chosen such that
\[
	G^2 = \frac{\sigma}{16(T-1)} \left(\sqrt{64 L \Delta (T-1)+9 \sigma ^2}-3 \sigma \right),
\]
and the functions $h_t$ are defined as follows: First denote by $h^{(1,L)}_{b,-}(x)$ and $h^{(1,L)}_{b,+}(x)$ the functions 
(see \figref{fig:hfunc})
\begin{align*}
& h^{(1,L)}_{b,+}(x) := \begin{cases}
    \frac{L}{2} x^2 &   |x| \leq b/4,\\
    \frac{L}{16} b^2  - \frac{L}{2} (|x|-b/2)^2 & b/4 < |x| < b/2, \\
    \frac{L}{16} b^2 & |x|\geq b/2,
\end{cases} \\
& h^{(1,L)}_{b,-}(x) := \begin{cases}
    0               &  |x| \leq b/2,\\
    \frac{L}{2} (|x|-\frac{b}{2})^2 &  \frac{b}{2} \leq |x| \leq \frac{3b}{4},\\
    \frac{L}{16} b^2  - \frac{L}{2} (|x|-b)^2 & \frac{3b}{4} < |x| < b, \\
    \frac{L}{16} b^2 & |x|\geq b,
\end{cases}
\end{align*}
then at indices $t$ where the aggregation coefficient $|\zeta_{t+1}| \leq \frac{1}{2}$, take $h_t$ to be $h_t=h^{(1,L)}_{|\eta_t| \sigma,-}(x)$, and otherwise take $h_t=h^{(1,L)}_{|\eta_t| \sigma,+}$.
Note that for all $t\in[T-1]$, 
\begin{align*}
    & h_t(0)=0, \\
    & h_t(x)=h_t(-x), \quad \forall x\in \reals, \\
    & h_t(\pm \eta_t \sigma) = \frac{L}{16} \eta_t^2 \sigma^2, \\
    & h_t'(0)=h_t'(\pm \eta_t \sigma)=h_t'(\pm \zeta_{t+1} \eta_t \sigma)=0,
\end{align*}
and that $h_t$ has $L$-Lipschitz gradient.
The purpose of the functions $h_t$ is to increase the value of $f(x_i)$ without affecting the gradient information available to the algorithm.

From the definition of $f$ we conclude that $f$ also shares the $L$-Lipschitz gradient of the functions $h_t$ (being a separable sum of functions with $L$-Lipschitz gradient).

\begin{figure}
\vskip 0.2in
\begin{center}
\centerline{\includegraphics[width=0.6\columnwidth]{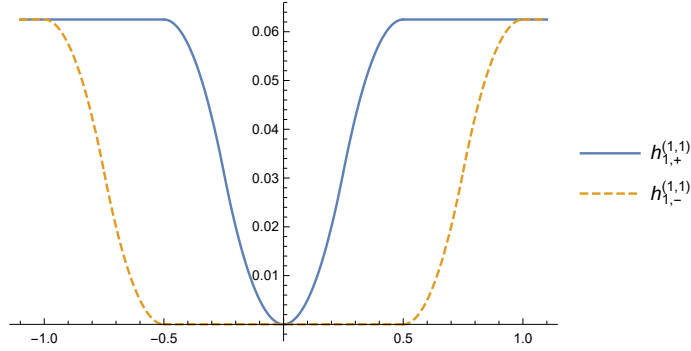}}
\caption{$h^{(1,1)}_{1,-}$ and $h^{(1,1)}_{1,+}$.}
\label{fig:hfunc}
\end{center}
\vskip -0.2in
\end{figure}

We now turn to analyze the dynamics of SGD when applied on the function $f$ defined above.
Given the objective $f$ and the starting point 
\[
    \bx_1=0,
\]
the algorithm at the first iteration sets
\[
    \bx_2 = \bx_1 - \eta_1 (\nabla f(\bx_1)+ \bxi_1) = (-\eta_1 G, \pm \eta_1 \sigma, 0, \dots, 0)^\top,
\]
hence, from the properties of $h_1$, we get
\begin{align*}
    f(\bx_2) 
    &= - G^2 \eta_1 + h_1(\eta_t \sigma), \\
    \nabla f(\bx_2) &= G \be_1 + h_1'(\pm \eta_1 \sigma)\be_2 = G \be_1.
\end{align*}
Similarly, at the $t$-th iteration, $t\in [T-1]$, the algorithm sets
\begin{equation}\label{eq:x_t_value}
\begin{aligned}
    \bx_{t+1} &= \bx_t - \eta_t (\nabla f(\bx_t)+ \bxi_t) ) \\ &= (-\sum_{k=1}^t \eta_k G, \pm \eta_1 \sigma, \dots, \pm \eta_t \sigma, 0, \dots, 0)^\top,
\end{aligned}
\end{equation}
which leads to
\begin{align}
    & f(\bx_t) = - G^2 \sum_{k=1}^{t-1} \eta_k + \sum_{k=1}^{t-1} h_k(\eta_k \sigma), \\
    & \nabla f(\bx_t) = G \be_1. \label{eq:nabla_F_x_t}
\end{align}

At the aggregation step, the algorithm sets
\begin{equation}\label{eq:x_out}
\begin{aligned}
    & \bxout = \sum_{t=1}^T \zeta_t \bx_t \\&= (-\sum_{t=1}^T \zeta_t \sum_{k=1}^{t-1} \eta_k G, \pm \zeta_2 \eta_1 \sigma, \dots, \pm \zeta_T \eta_{T-1} \sigma)^\top,
\end{aligned}
\end{equation}
then by the properties of $h_t$, we get
\begin{align}
    & f(\bxout) = - G^2 \sum_{t=1}^T \zeta_t \sum_{k=1}^{t-1} \eta_k + \sum_{k=1}^{t-1} h_k(\zeta_{k+1} \eta_k \sigma), \label{eq:F_xout} \\
    & \nabla f(\bxout) = G \be_1, \label{eq:nabla_F_xout}
\end{align}
where the 
first equality follows since $h_t$ is even, and second equality follows from $h_t'(\pm \zeta_{t+1} \eta_t \sigma)=0$.

To complete our treatment of the fixed-step case, we turn to show that it is possible to make $f$ bounded from below without affecting the first-order information at the iterates and the gradient at $\bxout$. For this purpose, we continue to show that \lemref{L:detbound} can be applied when taking for $\by_1,\dots,\by_m$ all the possible values the random variable $\bxout$~\eqref{eq:x_out} can attain, and for $\bz_1,\dots,\bz_n$ all possible values the random variables $\{\bx_t\}_{t\in [T]}$ can attain~\eqref{eq:x_t_value}. 

Indeed, in view of~\eqref{eq:nabla_F_x_t} and~\eqref{eq:nabla_F_xout}, the first condition of \lemref{L:detbound} holds with $\bgamma = G \be_1$, the second requirement follows from~\eqref{eq:F_xout}, and the third requirement follows since 
\[
    \langle \nabla f(\bxout), \bxout \rangle = -\sum_{t=1}^T \zeta_t \sum_{k=1}^{t-1} \eta_k G^2
\]
does not depend on the sign of the noise vectors $\bxi_t$.
As all the requirements of \lemref{L:detbound} hold, we conclude that there exists a function $\hat f$ that shares the the same first-order information as $f$ at $\{\bx_t\}_{t\in [T]}$, the same gradient at $\bxout$ and in addition
\[
    \inf_\bx \hat f(\bx) \geq \min_{t\in [T]} f(\bx_t) - \frac{3}{2L} G^2.
\]
We get
\begin{align*}
    & \hat f(\bx_1) - \inf_\bx \hat f(\bx) \\ 
    &\leq 0 - \min_{t\in [T]} f(\bx_t) + \frac{3}{2L} G^2 \\
    & = \max_{t\in [T]} \sum_{k=1}^{t-1} \left(\eta_k G^2 - \frac{L}{16} \eta_k^2 \sigma^2 + \frac{3}{2L(t-1)} G^2 \right)\\
    & \leq (T-1) \frac{G^2}{2 L}  \left(\frac{8 G^2}{\sigma ^2}+\frac{3}{T-1}\right) = \Delta,
\end{align*}
where the second inequality follows by maximizing the concave quadratic form over $\eta_k$ and the last inequality follows from the definition of $G^2$ by basic algebra.

As SGD does not have access to the objective beyond the first-order information at the iterates, we conclude that the algorithm proceeds on $\hat f$ in exactly the same dynamics as it does on $f$, maintaining its behavior as derived above.
\end{proof}

\medskip

The example provided by \thmref{thm:aggregation_step} comes with a guarantee that the gradient of the objective at all iterates is almost surely a constant; as a result, the theorem is applicable for forming lower bounds for all first-order optimality criteria used in the literature, including the best expected gradient norm $\min_t \E \|\nabla f(\bx_t)\|$, average expected gradient norm $\frac{1}{T} \sum_t \E \|\nabla f(\bx_t)\|$, and expected norm of the average gradient $\E \|\frac{1}{T} \sum_t \nabla f(\bx_t)\|$, both when taking the actual gradient and when taking the noisy version of the gradient.

Note that although the theorem does not directly consider randomized sampling schemes for computing $\bxout$, the performance of any scheme that samples $\bxout$ out of $\{\bx_1,\dots,\bx_T\}$ is bounded from below by the optimality criterion $\min_{t}\norm{\nabla f(\bx_t)}$, making the guarantees by the theorem applicable.

\begin{remark}[Tightness results]\label{rem:tight_bounds}
Consider the upper bound by Ghadimi and Lan (see~\thmref{T:upper_noncovex} in the appendix) and set the step size by
$
    \eta_t \equiv \eta:=\sqrt{\frac{2\Delta}{(T-1)L \sigma^2}},
$
where $\Delta$ is an upper bound on $f(\bx_1)-f(\bx_*)$. We obtain
\begin{align*}
    \min_{t \in [T]} \|\nabla f(\bx_t)\|^2 
    \leq& \frac{2\Delta + L (T-1) \eta^2 \sigma^2}{(T-1) \eta (2-L \eta)}
    \\ \underset{T\gg 1}{\approx}&
    \frac{2\Delta + L (T-1) \eta^2 \sigma^2}{2 (T-1) \eta}
    = \sigma \sqrt{\frac{2 L \Delta}{T-1}},
\end{align*}
which establishes on one hand, that the lower bound obtained in \thmref{thm:aggregation_step} on the iterates $\bx_t$ is tight up to the constant factor $2\sqrt{2}$,
and on the other hand, establishes that the constant step-size scheme defined above is optimal up to the same constant.
\end{remark}

The second main result of this subsection gives an $\Omega(\epsilon^{-3})$ lower bound on the performance of ``plain'' SGD methods (i.e., methods that do not perform an aggregation step) acting on objectives with Lipchitz Hessians.

\begin{theorem}\label{thm:nonconvex}
Consider a method that given a function $f:\reals^d\rightarrow\reals$ and an initial point $\bx_1\in \reals^d$ generates a sequence of $T\in\mathbb{N}$ points $\{\bx_t\}$ satisfying
\[
     \bx_{t+1} = \bx_t + \eta_{\bx_1,\dots,\bx_t} \cdot (\nabla f(\bx_t)+\bxi_t), \quad t\in[T-1],
\]
where $\bxi_t$ are some random noise vectors. We further assume that the step sizes $\eta_{\bx_1,\dots,\bx_t}$ are deterministic functions of the norms and inner products between $\bx_1,\dots,\bx_t, \nabla f(\bx_1)+\bxi_1,\dots, \nabla f(\bx_t)+\bxi_t$ and may also depend on the exact second-order information $\nabla^2 f(\bx_1),\dots, \nabla^2 f(\bx_t)$.
Then for any $\rho,\Delta,\sigma \in \reals_{++}$ there exists a function $f:\reals^{T}\mapsto \reals$ with $\rho$-Lipschitz Hessian, $\bx_1\in \reals^T$, and independent random variables $\bxi_t$ with $\E[\bxi_t]=0$ and $\E[\|\bxi_t\|^2]=\sigma^2$ such that $\forall t\in[T]$
    \begin{align*}
        & f(\bx_1) - f(\bx_t) \overset{\text{a.s.}}{\leq} \Delta,
        & \|\nabla f(\bx_t)\|^2 \overset{\text{a.s.}}{=} \bgamma,
    \end{align*}
where $\bgamma\in \reals^T$ is a vector that satisfies
\[
	\|\bgamma\|^2 = \frac{\sigma}{2}\left(\frac{\rho \Delta ^2}{ (T-1)^2}\right)^{1/3}.
\]
\end{theorem}

\begin{proof}
We proceed as in the proof of \thmref{thm:aggregation_step}, taking for $G$ the positive value that satisfies
\[
G^2 
    = \frac{3 \sigma}{32} \left(\frac{16^2 \rho \Delta ^2}{ (T-1)^2}\right)^{1/3}
    \geq \frac{\sigma}{2}\left(\frac{\rho \Delta ^2}{ (T-1)^2}\right)^{1/3}, \\
\]
and set $h_t:=h^{(2,\rho)}_{|\eta_t| \sigma}$, with $h^{(2,\rho)}_{b}$ defined by
(see \figref{fig:sfunc2}):
\begin{align*}
    & h^{(2,\rho)}_{b}(x) :=\\& \begin{cases}
        \frac{\rho}{6} |x|^3 &   |x| \leq b/4,\\
        \frac{\rho}{2} \left(\frac{b^3}{96}-\frac{b^2}{8}  |x|+\frac{b}{2}  x^2-\frac{1}{3}|x|^3 \right) &   \frac{b}{4} \leq |x| < \frac{3b}{4},\\
        \frac{\rho}{32} b^3 - \frac{\rho}{6} \left(b - |x| \right)^3 & \frac{3b}{4} \leq |x| < b, \\
        \frac{\rho}{32} b^3 & |x|\geq b.
    \end{cases}
\end{align*}
It is straightforward to verify that $h_t$ has $\rho$-Lipschitz Hessian, and as in the Lipschitz gradient case, we have
\begin{align*}
    & h_t(0)=0, \\
    & h_t(x)=h_t(-x), \quad \forall x\in \reals, \\
    \text{and}\quad & h_t'(0)=h_t'(\eta_t \sigma)=h_t'(-\eta_t \sigma)=0.
\end{align*}
Proceeding with the new values, we reach
\begin{align*}
    & f(\bx_1) - f(\bx_t) = \sum_{k=1}^{t-1} \left(G^2 \eta_k - h^{(2)}_{|\eta_k| \sigma}(|\eta_k| \sigma)\right) \\
    & = \sum_{k=1}^{t-1} \left( \frac{3 \sigma}{32}\left(\frac{16^2 \rho \Delta ^2}{ (T-1)^2}\right)^{1/3} \eta_k - \frac{\rho}{32} |\eta_k|^3 \sigma^3\right) \\
    & = \frac{1}{32} \sum_{k=1}^{t-1} \left(3 \left(\frac{16^2 \Delta ^2} {(T-1)^2}\right)^{1/3} (\rho^{1/3} \sigma \eta_k) - |\rho^{1/3} \eta_k \sigma|^3\right) \\
    & \leq (t-1) \frac{\Delta}{T-1} \leq \Delta,
\end{align*}
where the one before last inequality follows from the inequality $3 a x - |x|^3 \leq 2 a^{\frac{3}{2}}$.
    
Finally, note that $h''_t (\eta_t \sigma) = h''_t (-\eta_t \sigma) = 0$, and as a result, the second-order information of $f$ at all iterates is identically zero, thus the proof in the adaptive step-size case can proceed without change.
\end{proof}

\begin{figure}
    \centering
    \includegraphics[width=0.6\linewidth]{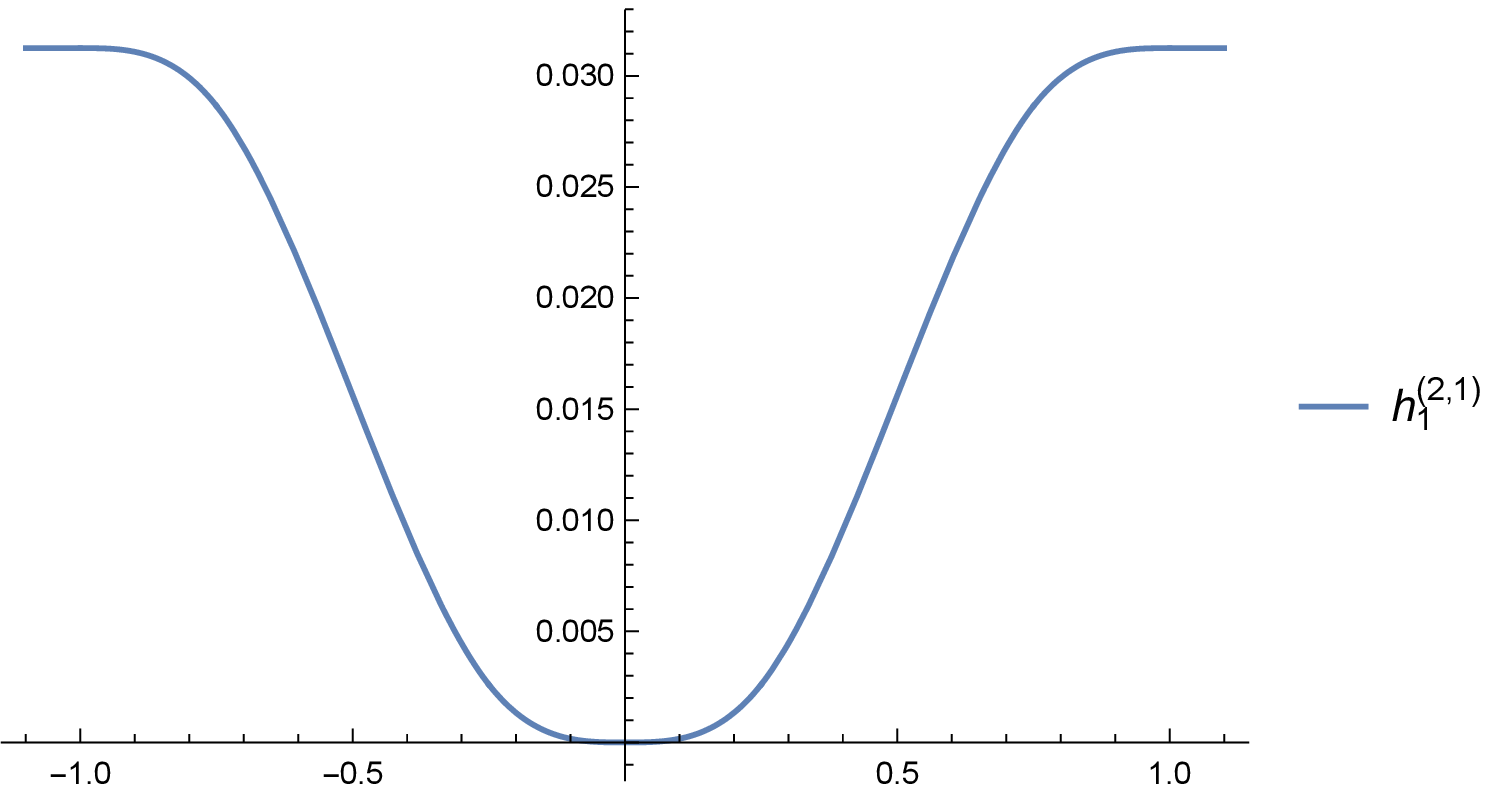}
    \caption{$h^{(2,1)}_{1}$.}
    \label{fig:sfunc2}
\end{figure}

Note that the main missing component needed for establishing a result bounding $f(\bx_1) - \inf_\bx f(\bx)$ in the Lipschitz-Hessian case is a set of necessary and sufficient interpolation conditions for Lipschitz-Hessian functions (as in the case of Lipschitz-gradient, \thmref{T:nonconvex_interpolation} in the appendix). The existence of such conditions remains an open question.

\section{Lower bounds in the convex quadratic case}\label{sec:convex}

In this section, we continue our analysis of the SGD method, showing that even for convex, quadratic functions in moderate dimensions and a standard Gaussian noise, SGD cannot achieve an iteration complexity better than $\Ocal(\epsilon^{-4})$ in order for any of its iterates to have gradient norm less than $\epsilon$. Note that for quadratic functions, the Hessian is constant, so the result still holds under additional 
Lipschitz assumptions on the Hessian and higher-order derivatives. We emphasize that the lower bounds only hold for the iterates themselves, without any aggregation step. Formally, we have the following:

\begin{theorem}\label{thm:sgdlow}
Consider the SGD method defined by
\[
     \bx_{t+1} = \bx_t + \eta_t \cdot (\nabla f(\bx_t)+\bxi_t),  \quad t\in [T-1],
\]
for some $T>1$ and suppose that the step sizes $\eta_1,\ldots,\eta_{T-1}$ are non-negative and satisfy 
at least one of the 
following conditions:
\begin{enumerate}
\item \emph{(Small step sizes)} $\max_{t\in [T-1]} \eta_t\leq 1/L$, and $\sum_{t=1}^{T-1}\eta_t \leq c\sqrt{T}/L$
for some constant $c$ (independent of the problem parameters). 
\item \emph{(Fixed step sizes)} $\eta_t$ is the same for all $t$.
\item \emph{(Polynomial decay schedule)} $\eta_t = \frac{a}{b+t^{\theta}}$ 
for 
some non-negative constants $a,b,\theta$ (independent of the problem parameters).
\end{enumerate}
Then for any $\delta\in (0,1)$, there exists a \emph{quadratic} function $f$ on $\reals^d$ (for any $d\geq d_0$ with $d_0=\Ocal(\log(T/\delta)\sigma^2T/(L^2\Delta))$ with $L$-Lipschitz gradients, 
and $\bx_1$ for which $f(\bx_1)-\inf_{\bx}f(\bx)\leq \Delta$, such that if $\bxi_t$ 
has a Gaussian distribution 
$\Ncal(\mathbf{0},\frac{\sigma^2}{d}I_d)$, with probability at least $1-\delta$
\[
\min_{t\in [T-1]}\norm{\nabla f(\bx_t)}^2 \geq c_0 \frac{\min\{L \Delta, \sigma^2\}}{\sqrt{T}},
\]
where $c_0$ is a positive constant depending only on the constants in the conditions stated above.
\end{theorem}

We note that all standard analyses for (non-adaptive) SGD methods rely on one of these step-size strategies. Moreover, the proof technique can plausibly be extended to other step sizes. Thus, the theorem provides a strong indication that SGD (without an aggregation step) cannot achieve a better iteration complexity, at least when the optimality criterion is $\min_t\norm{\nabla f(\bx_t)}$, even for convex quadratic functions.

The proof is based on the following two more technical propositions, which provide lower bounds depending on the step sizes and the problem parameters: 

\begin{proposition}\label{prop:distance}
	For any $L>0,\Delta>0,T>1$ and $\delta\in (0,1)$, there exists a convex quadratic function $f$ on $\reals^d$ (for any $d\geq d_0$ where $d_0=\Ocal(\log(T/\delta)\sigma^2T/(L^2\Delta))$) with $L$-Lipschitz gradient, and an $\bx_1$ such that $f(\bx_1)-\inf_{\bx}f(\bx)\leq \Delta$, such that if we initialize SGD at $\bx_1$ with Gaussian noise $\Ncal(\mathbf{0},\frac{\sigma^2}{d}I_d)$ and use step sizes $\eta_1,\ldots,\eta_{T-1}$ in $[0,1/L]$, then with probability at least $1-\delta$,
\begin{align*}
	\min_{t\in [T]}\norm{\nabla 
	f(\bx_t)}^2 \geq \frac{\Delta}{25 \max\left\{1/L,\sum_{t=1}^{T-1}\eta_t\right\}}.
\end{align*}
\end{proposition}

\begin{proposition}\label{prop:noise}
	For any $L>0,\Delta>0,T>1$ and $\delta\in (0,1)$, there exists a convex quadratic function $f$ on $\reals^d$ (for any $d\geq d_0$ where $d_0=\Ocal(\log(T/\delta))$) with $L$-Lipschitz gradient, and a vector $\bx_1$ such that $f(\bx_1)-\inf_{\bx}f(\bx)\leq \Delta$, such that if we initialize SGD at $\bx_1$ with Gaussian noise $\Ncal(\mathbf{0},\frac{\sigma^2}{d}I_d)$, then the following holds with probability at least $1-\delta$:
	\begin{itemize}
		\item If for all $t$, $\eta_t=\eta$ with $\eta\in [0,1/L)$, then 
		$
		\min_{t\in [T]}\norm{\nabla f(\bx_t)}^2 \geq \frac{L}{2} \min\{\Delta,\frac{\eta \sigma^2}{2-L\eta}\}.
		$
		\item If for all $t$, $\eta_t \geq c/L$ for some constant $c>0$ then $\min_{t\in [T]}\norm{\nabla f(\bx_t)}^2 \geq \frac{\sigma^2 c^2}{2}$.
		\item If $\eta_t = \frac{a}{L(b+t^\theta)}$ for some positive constants $a>0,b\geq 0$ and
		$\theta\in (0,1)$, then 
		\[
		\min_{t\in [T]}\norm{\nabla f(\bx_t)}^2 \geq c_{a,b,\theta} \sigma^2 \min\{1, L \eta_T\},
		\]
		where $c_{a,b,\theta}$ is a constant dependent only on $a,b,\theta$.
	\end{itemize}
\end{proposition}
The proofs of these propositions appear in \appref{sec:proofs}. Together, Propositions \ref{prop:distance} and \ref{prop:noise} imply the theorem:
\begin{proof}[Proof of \thmref{thm:sgdlow}]
The theorem, under the first condition, is an immediate corollary of \propref{prop:distance}. Indeed,
\begin{align*}
    \min_{t\in [T]} \norm{\nabla 
	f(\bx_t)}^2
	\geq \frac{\Delta}{25\max\left\{1/L,\sum_{t=1}^{T-1}\eta_t\right\}}& \\
	\geq \frac{L \Delta}{25\max\left\{1,c \sqrt{T}\right\}} 
	\geq \frac{\min\{L \Delta, \sigma^2\}}{25\max\{1, c\} \sqrt{T}}&.
\end{align*}
As to the second condition, let us consider three cases. First, if $\eta_t=\eta$ is at most $T^{-1/2}/L$, then $\sum_{t=1}^{T-1}\eta_t< \sqrt{T}/L$ and the result again follows from \propref{prop:distance} as in the previous case. Next, suppose $T^{-1/2}/L \leq \eta < 1/L$, then the result follows from \propref{prop:noise}:
\begin{align*}
    & \min_{t\in [T]} \norm{\nabla f(\bx_t)}^2 \geq \frac{L}{2} \min\{\Delta, \frac{ \eta \sigma^2}{2-L\eta}\}
    \\&\geq \frac{1}{2} \min\{L \Delta, \frac{\sigma^2 T^{-1/2}}{2-T^{-1/2}}\}
    \geq \frac{1}{2} \frac{\min\{L \Delta,\sigma^2\}}{2\sqrt{T}-1}.
\end{align*}
The last case for this condition is $1/L \leq \eta$, which does not converge due to the second case in \propref{prop:noise}.

As to the third condition (namely $\eta_t = \frac{a}{b+t^{\theta}}$), we can assume without loss of generality that 
$a>0$ (otherwise we are back to the first condition in the theorem, and nothing is left to prove) and that $\theta \in (0,1)$ (since if 
$\theta=0$, we are back to the second condition in the theorem and if $\theta \geq 1$ we are back to the first condition in the theorem). The result then follows from the third case of Proposition \ref{prop:noise}.
\end{proof}


\bibliographystyle{plain}
\bibliography{bib}

\appendix

\section{Additional Proofs}\label{sec:proofs}

\subsection{Bounded smooth interpolation}

We start by recalling a recent and fundamental theorem which provide necessary and sufficient conditions under which a set $\{\bx_t, \bg_t, f_t\}_{i\in [T]}$ can be \emph{interpolated} (or \emph{extended} using the terminology from the classical text~\cite{whitney1934analytic}) by a convex function $f$ with $L$-Lipschitz gradient such that $f(\bx_t) =f_t$ and $\nabla f(\bx_t) = \bg_t$ for all $t\in [T]$.
The theorem was established in~\cite{taylor2015smooth} and also independently in~\cite{azagra2017extension} for a more general setting in Hilbert spaces. 
\begin{theorem}\label{T:convex_interpolation}
Let $L>0$, $d\in \mathbb{N}$ and suppose $\{(\bx_t, \bg_t, f_t)\}_{t\in [T]}$ is some finite subset of $\reals^d\times\reals^d\times \reals$. Then there exists a convex function $f$ with $L$-Lipschitz gradient that satisfies $f(\bx_t) = f_t$ and $\nabla f(\bx_t) = \bg_t$ for all $t\in [T]$ if and only if
\begin{equation}\label{E:convex_interpolation_conditions}
	    \frac{1}{2L} \|\bg_i - \bg_j\|^2 \leq f_i - f_j - \langle \bg_j, \bx_i - \bx_j\rangle, \quad \forall i,j \in [T].
\end{equation}
\end{theorem}

A similar result that provides necessary and sufficient conditions for non-convex interpolation is also known.
\begin{theorem}[{\cite[Theorem~3.10]{taylor2015exact}}]\label{T:nonconvex_interpolation}
Let $L>0$, $d\in \mathbb{N}$ and suppose $\{(\bx_t, \bg_t, f_t)\}_{t\in [T]}$ is some finite subset of $\reals^d\times\reals^d\times \reals$. Then there exists an function $f$ with $L$-Lipschitz gradient that satisfies $f(\bx_t) = f_t$ and $\nabla f(\bx_t) = \bg_t$ for all $t\in [T]$ if and only if
\begin{equation}\label{E:nonconvex_interpolation_conditions}
	    \frac{1}{2L} \|\bg_i - \bg_j\|^2 - \frac{L}{4} \|\bx_i - \bx_j - \frac{1}{L}(\bg_i - \bg_j)\|^2 \leq f_i - f_j - \langle \bg_j, \bx_i - \bx_j\rangle, \quad \forall i,j \in [T].
\end{equation}
\end{theorem}

Here we strengthen the results of \thmref{T:convex_interpolation} and \thmref{T:nonconvex_interpolation}, showing that interpolation can be performed in such a way that the resulting function is bounded from below and attains its minimum value. The proof is based on an explicit construction of a convex interpolating function developed in~\cite{drori2017exact}.
This resolves an open question raised by~\cite{fefferman2017interpolation} for the case where the interpolation set is finite.

\begin{theorem}\label{T:nonnegative_interpolation}
Let $L>0$, $d\in \mathbb{N}$ and suppose $\{(\bx_t, \bg_t, f_t)\}_{t\in [T]}$ is some finite subset of $\reals^d\times\reals^d\times \reals$ that satisfies~\eqref{E:convex_interpolation_conditions} (alternatively, \eqref{E:nonconvex_interpolation_conditions}).
Then there exists a convex (alternatively, nonconvex) function $f$ with $L$-Lipschitz gradient that satisfies $f(\bx_t) = f_t$, $\nabla f(\bx_t) = \bg_t$ and in addition, setting $j \in \arg\min_{t\in [T]} f_t -\frac{1}{2L} \| \bg_t\|^2$, the function $f$ also satisfies
\[
    f^*:= \min_{x\in \reals^d} f(x) = f(\bx_j - \frac{1}{L} \bg_j) = f_j - \frac{1}{2L} \|\bg_j\|^2.
\]
\end{theorem}

\begin{proof}
The convex case follows directly from Theorem~1 in~\cite{drori2017exact}. Indeed, taking $C\leftarrow\{0\}$ and
$\mathcal{T}\leftarrow \{(\bx_t, \bg_t, f_t)\}_{t\in [T]}$,
the primal interpolation function $\WTC$ (see \cite[Definition~2.1]{drori2017exact}) can be written as
\begin{equation}\label{eq:zdefconvex}
    W(y) :=  \min_{\balpha\in \Delta_{T}} \left[ \frac{L}{2} \|y - \sum_{t\in T} \alpha_t  (\bx_t-\frac{1}{L} \bg_t) \|^2 + \sum_{t\in T} \alpha_t (f_t - \frac{1}{2L} \| \bg_t \|^2)\right],
\end{equation}
where $\Delta_{T}$ is the $T$-dimensional unit simplex
\[
  \Delta_{T}:=  \{\balpha \in \reals^T : \sum_{t\in T} \alpha_t = 1, \ \alpha_t \geq 0, \ \forall t\in T \}.
\]
By the assumption that $\mathcal{T}$ satisfies~\eqref{E:convex_interpolation_conditions},
Theorem~1 in~\cite{drori2017exact} implies that $W$ is convex, its gradient is $L$-Lipschitz and that
$W(\bx_t) = f_t$, $\nabla W(\bx_t) = \bg_t$.
The lower bound on $W$ then immediately follows from~\eqref{eq:zdefconvex}, as
\[
    W(y) \geq \min_{\balpha\in \Delta_{T}} \left[\sum_{t\in T} \alpha_t (f_t - \frac{1}{2L} \| \bg_t \|^2)\right] = \min_{t\in [T]} (f_t -\frac{1}{2L} \| \bg_t\|^2) = f_j -\frac{1}{2L} \| \bg_j\|^2,
\]
and
\[
    W(\bx_j - \frac{1}{L} \bg_j) \leq f_j -\frac{1}{2L} \| \bg_j\|^2, 
\]
which follows by taking $\balpha= \be_j$ in~\eqref{eq:zdefconvex}.
Finally, combining these two bounds completes the proof for the convex case:
\[
    f_j -\frac{1}{2L} \| \bg_j\|^2 \leq \inf_y W(y) \leq  W(\bx_j - \frac{1}{L} \bg_j) = f_j -\frac{1}{2L} \| \bg_j\|^2.
\]

For the non-convex case, consider the function
\begin{align*}
    Z(y) :=  \min_{\balpha\in \Delta_{T}} &\left[ L \|y - \sum_{t\in T} \alpha_t  (\bx_t-\frac{1}{2L}( \bg_t + L \bx_t)) \|^2 + \sum_{t\in T} \alpha_t (f_t + \frac{L}{2}\|\bx_t\|^2 -\frac{1}{4L} \| \bg_t + L \bx_t\|^2)\right].
\end{align*}
It is straightforward to verify that this function is
the primal interpolation function $\WTC$ taking $C\leftarrow\{0\}$, $L\leftarrow 2L$, and
\[
    \mathcal{T}\leftarrow
    \{(\bx_t, \bg_t + L \bx_t, f_t + \frac{L}{2}\|\bx_t\|^2)\}_{t\in [T]}.
\]
As $\mathcal{T}$ satisfies~\eqref{E:convex_interpolation_conditions} with a Lipschitz constant $2L$,
by Theorem~1 in~\cite{drori2017exact} it follows that $Z$ is convex, has a $2L$-Lipschitz gradient and satisfies
\begin{align*}
    & Z(\bx_t) = f_t + \frac{L}{2}\|\bx_t\|^2, \\
    & \nabla Z(\bx_t) = \bg_t + L \bx_t.
\end{align*}
Now let $\hat W$ be defined by $\hat W(y) := Z(y) - \frac{L}{2} \|y\|^2$. Clearly, $\hat W$ has $L$-Lipschitz gradient (see e.g., \cite[Lemma~3.9]{taylor2015exact}), satisfies 
\begin{align*}
    & \hat W(\bx_t) = f_t, \\
    & \nabla \hat W(\bx_t) = \bg_t,
\end{align*}
and by basic algebra it is straightforward to show that
\begin{equation}\label{eq:zdef}
\begin{aligned}
    \hat W(y) = \min_{\balpha\in \Delta_{T}} & \left[ \frac{L}{2} \|y - \sum_{t\in [T]} \alpha_t  (\bx_t-\frac{1}{L} \bg_t) \|^2 - \frac{L}{4}\|\sum_{t\in [T]} \alpha_t  (\bx_t-\frac{1}{L} \bg_t) \|^2 
    \right. \\ & \left.
    + \sum_{t\in [T]} \alpha_t (f_t -\frac{1}{2L} \| \bg_t\|^2 + \frac{L}{4} \| \bx_t - \frac{1}{L} \bg_t\|^2)\right].
\end{aligned}
\end{equation}
We have
\begin{align*}
    \hat W(y)
    & \geq \min_{\balpha\in \Delta_{T}} - \frac{L}{4} \sum_{t\in [T]} \| \alpha_t ( \bx_t-\frac{1}{L} \bg_t )\|^2 + \sum_{t\in [T]} \alpha_t (f_t -\frac{1}{2L} \| \bg_t\|^2 + \frac{L}{4} \| \bx_t - \frac{1}{L} \bg_t\|^2) \\
	& \geq \min_{\balpha\in \Delta_{T}} - \frac{L}{4} \sum_{t\in [T]} \alpha_t  \|\bx_t-\frac{1}{L} \bg_t\|^2 + \sum_{t\in [T]} \alpha_t (f_t -\frac{1}{2L} \| \bg_t\|^2 + \frac{L}{4} \| \bx_t - \frac{1}{L} \bg_t\|^2) \\
	& = \min_{\balpha\in \Delta_{T}} \sum_{t\in [T]} \alpha_t (f_t -\frac{1}{2L} \| \bg_t\|^2) = \min_{t\in [T]} (f_t -\frac{1}{2L} \| \bg_t\|^2) = f_j -\frac{1}{2L} \| \bg_j\|^2,
\end{align*}
where the second inequality follows from the convexity of the squared norm.
Finally we conclude the proof by establishing an upper bound that matches the lower bound on $\hat W^*$. Indeed,
\begin{align*}
\hat W (\bx_j - \frac{1}{L} \bg_j)  \leq f_j -\frac{1}{2L} \| \bg_j\|^2,
\end{align*}
where the inequality follows, as in the convex case, by taking $\balpha= \be_j$ in~\eqref{eq:zdef}.
\end{proof}

\subsection{Proof of \lemref{L:detbound}}\label{S:lemma_proof}

Theorem~\ref{T:nonnegative_interpolation} will be the main tool used in the proof of \lemref{L:detbound} below.

\begin{proof}[Proof of \lemref{L:detbound}]
By Theorem~\ref{T:nonnegative_interpolation}, it is sufficient to show that there is a choice $\beta$ for the value of $\hat f(\by_1),\dots,\hat f(\by_m)$ with $\beta \geq \min_{i \in [n]} f(\bz_{i}) - \frac{1}{L} \|\bgamma\|^2$ such that the set
\[
     \{(\by_j, \bgamma, \beta)\}_{j\in [m]} \cup \{(\bz_{i}, \bgamma, f(\bz_i) \}_{i\in [n]},
\]
satisfies the interpolation conditions~\eqref{E:nonconvex_interpolation_conditions}. 
The lower bound on $\hat f$ will then immediately follow since
\begin{align*}
    && \hat f(\by_j) - \frac{1}{2L}\|\bgamma\|^2 = \beta - \frac{1}{2L} \|\gamma\|^2 \geq \min_{k\in[n]} f(\bz_k) - \frac{3}{2L} \|\bgamma\|^2, \quad \forall j \in [m], \\
\text{and}&&    \hat f(\bz_i) - \frac{1}{2L}\|\bgamma\|^2 \geq \min_{k\in[n]} f(\bz_k) - \frac{3}{2L} \|\bgamma\|^2, \quad \forall i \in [n].
\end{align*}

In order to establish that the interpolation conditions hold, first note that all the interpolation conditions involving two points from $\{\bz_i\}$ are naturally satisfied by the assumption that there exists some function with $L$-Lipschitz gradient (namely $f$) that interpolates $\{(\bz_i, \nabla f(\bz_i), f(\bz_i) \} = \{(\bz_i, \bgamma, f(\bz_i) \}$,
and further note that by the assumptions, it follows that $\langle \bgamma, \by_i - \by_j\rangle=0$, hence the interpolation conditions involving both points in $\by_i$ are also trivially satisfied. We conclude that we only need to consider~\eqref{E:nonconvex_interpolation_conditions} for cases where one of the points is $\bz_i$ and the other is $\by_j$, i.e., we are left the following set of inequalities:
\begin{align*}
    & - \frac{L}{4} \|\bz_{i} - \by_j \|^2 \leq f(\bz_i) - \beta - \langle \bgamma, \bz_i - \by_j \rangle, \quad i\in [n],\ j\in [m], \\
    & - \frac{L}{4} \|\by_j - \bz_{i} \|^2 \leq \beta - f(\bz_{i}) - \langle \bgamma, \by_j - \bz_{i}\rangle, \quad i\in [n],\ j\in [m].
\end{align*}
Clearly, these inequalities hold if and only if
\begin{align*}
    & \beta \in [\max_{i,j} \left( f(\bz_{i}) + \langle \bgamma, \by_j - \bz_{i}\rangle - \frac{L}{4} \|\by_j - \bz_{i} \|^2\right), \min_{i,j} \left( f(\bz_{i}) + \langle \bgamma, \by_j - \bz_{i}\rangle + \frac{L}{4} \|\bz_{i} - \by_j \|^2 \right) ].
\end{align*}
Now, this range is non-empty since it contains $f(\by_1)$ (recall that $f(\by_1)=\dots=f(\by_m)$, $\nabla f(\by_1)=\dots=\nabla f(\by_m) = \bgamma$, and that the interpolation conditions for the set $\{(\by_j, \nabla f(\by_j), f(\by_j)\} \cup \{(\bz_i, \nabla f(\bz_i), f(\bz_i) \}$ naturally hold), hence there exits some $i,j$ such that the choice
\[
    \hat \beta := f(\bz_{i}) + \langle \bgamma, \by_j - \bz_{i}\rangle + \frac{L}{4} \|\bz_{i} - \by_j \|^2
\]
is a feasible choice for $\beta$.
We get
\begin{align*}
    \hat \beta 
    & = f(\bz_{i}) + \frac{L}{4} \| \bz_{i} - \by_j - \frac{2}{L} \bgamma \|^2 - \frac{1}{L} \|\bgamma\|^2 \\
    & \geq \min_{k} f(\bz_{k}) - \frac{1}{L} \|\bgamma\|^2.
\end{align*}
which concludes the proof, as all interpolation conditions are satisfied, hence a function with the claimed properties exists.
\end{proof}

\subsection{Proof of \thmref{thm:aggregation_step}, adaptive step-size case}
Consider the general, adaptive step-size case:
\begin{align*}
    & \bx_{t+1} = \bx_t + \eta_{\bx_1,\dots,\bx_t} \cdot (\nabla f(\bx_t)+\bxi_t), \quad t\in [T-1], \\
    & \bxout = \sum_{t=1}^T \zeta^{(t)}_{\bx_1,\dots,\bx_T} \bx_i.
\end{align*}
Here, our goal is to show that constants $\eta_t$ and $\zeta_t$ from the proof of the fixed step-size case can be chosen in such a way that the method, when applied on $f_{\{\eta_t\},\{\zeta_t\}}$ constructed above, chooses step sizes and aggregation coefficients that are almost surely equal to the selected constants, i.e,
\[
    \eta_{\bx_1,\dots,\bx_t} \overset{\text{a.s.}}{=} \eta_t, \quad \zeta^{(t)}_{\bx_1,\dots,\bx_T} \overset{\text{a.s.}}{=} \zeta_t,
\]
and thus the proof for the fixed-step case can proceed without change.

We use the following procedure to select $\eta_t$ and $\zeta_t$.
We start by executing the first step of the algorithm on the initial point $\bx_1=0$ and $f$, where the constants $\{\eta_t\}$ and $\{\zeta_t\}$ are set to arbitrarily values. Note that 
\begin{enumerate*}
    \item the first-order information of $f$ at $\bx_1$ is independent of the choice for $\eta_t, \zeta_t$, and
    \item the norm of the noise vector $\bxi_1$ and its inner product with $\bx_1$ and $\nabla f(\bx_1)$ are independent of the specific direction chosen for the noise; therefore, by the assumption on the step size $\eta_{\bx_1}$, it is independent of the specific value for $\bxi_1$, i.e., it is almost surely a constant.
\end{enumerate*}
We denote this constant by $\eta_1$.

We continue by executing the second step of the algorithm on $f$, using the value of $\eta_1$ chosen above while keeping the constants $\eta_2,\dots,\eta_{T-1}$, $\{\zeta_t\}_{t\in[T]}$ set to arbitrary values. As in the first iteration,
\begin{enumerate*}
    \item the first-order information of $f$ at $\bx_2$ is independent of the specific choice for $\eta_2,\dots,\eta_{T-1}$, $\{\zeta_t\}_{t\in[T]}$ and
    \item the norm of the noise vector $\bxi_2$ and its inner product with $\bx_1, \bx_2$, $\nabla f(\bx_1), \nabla f(\bx_2)$ and $\bxi_1$ are independent of the specific direction chosen for the noise; therefore by the assumption on $\eta_{\bx_1,\bx_2}$ it is almost surely a constant.
\end{enumerate*}
As before, we denote the step size performed by the algorithm $\eta_2$.

Continuing in this fashion, we obtain a set of constants $\eta_1,\dots,\eta_{T-1}$ with the property that
when applying the method on $f=f_{\{\eta_t\},\{\zeta_t\}}$, then for any choice of aggregation coefficients $\{\zeta_t\}$ the step-sizes chosen by the method are almost surely equal to $\{\eta_t\}$.
Finally, executing the aggregation step, by the assumption on the aggregation function, the coefficients are almost surely constants, which we denote by $\zeta_1,\dots,\zeta_{T}$. To conclude, we have found a function $f=f_{\{\eta_t\},\{\zeta_t\}}$ such that the step sizes performed by the method on $f$ are almost surely $\eta_1,\dots,\eta_{T-1}$ and the aggregation coefficients chosen by the method are almost surely $\zeta_1,\dots,\zeta_{T}$,
hence the proof can continue as in the fixed-step case.

\subsection{Proof of Proposition \ref{prop:distance}}

We will utilize the 
following function:
\[
f(\bx) = \frac{1}{4\max\left\{1/L,\sum_{t=1}^{T-1}\eta_t\right\}}\cdot \langle \bx, \be_1\rangle^2
\]
and assume that the initialization $\bx_1$ is
\[
\bx_1 := 
\left(\sqrt{\Delta\cdot\max\left\{1/L,\sum_{t=1}^{T-1}\eta_t\right\}}~,~0,0,\ldots,0\right)~.
\]
It is easily verified that $f$ has $L$-Lipschitz 
gradient, and that 
$f(\bx_1)-\inf_{\bx} f(\bx)<\Delta$. Moreover, 
\begin{equation}\label{eq:norm_f_distance}
    \norm{\nabla 
	f(\bx)}=\frac{|\langle \bx, \be_1\rangle|}{2\max\left\{1/L,\sum_{t=1}^{T-1}\eta_t\right\}}.
\end{equation} 

Hereafter, for the sake of simplicity we 
drop the subscript indicating the coordinate number, and let $x_t$ denote the 
first coordinate of iterate $t$, and $\xi_t$ the first coordinate of the noise 
at iteration~$t$.

We now turn to show that when $d$ is large enough
\begin{equation}\label{eq:toshoww}
\min_{t\in [T]}|x_t|\geq 
\frac{2}{5}\sqrt{\Delta\max\left\{1/L,\sum_{t=1}^{T-1}\eta_t\right\}}
\end{equation}
holds with arbitrarily high probability, which together with~\eqref{eq:norm_f_distance} implies the desired result.

The dynamics of SGD on the first coordinate is as follows: we initially have \[
    x_1 = \sqrt{\Delta\cdot\max\{1/L,\sum_{t=1}^{T-1}\eta_t\}},
\]
and
\[
x_{t+1} = 
\left(1-\frac{\eta_t}{2\max\left\{1/L,\sum_{t=1}^{T-1}\eta_t\right\}}\right)x_t
-\eta_t
\xi_{t}.
\]
Unrolling this recurrence, we have for any $t$
\begin{equation}\label{eq:xtexpression}
\begin{aligned}
    x_t  & = \sqrt{\Delta\cdot\max\left\{1/L,\sum_{t=1}^{T-1}\eta_t\right\}} \cdot \prod_{j=1}^{t-1}\left(1-\frac{\eta_j}{2\max\left\{1/L,\sum_{t=1}^{T-1}\eta_t\right\}}\right) \\
    &\qquad - \sum_{j=1}^{t-1}\eta_j \xi_j  \prod_{i=j+1}^{t-1}\left(1-\frac{\eta_j}{2\max\left\{1/L,\sum_{t=1}^{T-1}\eta_t\right\}}\right),
\end{aligned}
\end{equation}
where we use the convention that $\prod_{i=a}^{b}c_i$ is always $1$ if $b<a$.
Since each $\xi_j$ is a zero-mean independent Gaussian, $x_t$ is also Gaussian with
\begin{align*}
\E [ x_t ] &= \sqrt{\Delta\cdot\max\left\{1/L,\sum_{t=1}^{T-1}\eta_t\right\}} \cdot \prod_{j=1}^{t-1}\left(1-\frac{\eta_j}{2\max\left\{1/L,\sum_{t=1}^{T-1}\eta_t\right\}}\right) \\
& \geq \sqrt{\Delta\cdot\max\left\{1/L,\sum_{t=1}^{T-1}\eta_t\right\}}\cdot 
\exp\left(\ln \frac{1}{2}\cdot \sum_{j=1}^{t-1}\frac{ \eta_j}{\max\left\{1/L,\sum_{t=1}^{T-1}\eta_t\right\}}\right) \\
&\geq
\frac{1}{2} \sqrt{\Delta\cdot\max\left\{1/L,\sum_{t=1}^{T-1}\eta_t\right\}},
\end{align*}
here we used the assumption $\eta_t \geq 0$ and the fact that $1-z/2\geq \exp(\ln \frac{1}{2} \cdot z)$ for all $z\in [0,1]$. In addition,
\begin{align*}
    \Var [x_t] 
    & = \sum_{j=1}^{t-1}\eta_j^2 \Var[\xi_j] \prod_{i=j+1}^{t-1}\left(1-\frac{\eta_j}{2\max\left\{1/L, \sum_{t=1}^{T-1}\eta_t\right\}}\right)^2 \\
    & \leq \sum_{j=1}^{t-1}\eta_j^2 \Var [\xi_j]  \leq \frac{\sigma^2 (T-1)}{L^2d},
\end{align*}
which follows since each $\xi_j$ is independent and with variance at most $\sigma^2/d$, and $0\leq \eta_j \leq 1/L$. Choosing 
\[
    d \geq d_0 := \frac{ \Phi^{-1}(1-\delta/T)^2 \sigma^2 (T-1)}{\left(\frac{1}{2} - \frac{2}{5} \right)^2 L^2 \Delta\cdot\max\left\{1/L,\sum_{t=1}^{T-1}\eta_t\right\}} = \Ocal(\log(T/\delta)\sigma^2T/(L^2\Delta)),
\]
where $\Phi^{-1}$ is the inverse CDF of the normal distribution, we get that for all $t$ with $\Var [x_t]>0$
\begin{align*}
    & \Pr \left(x_t \geq \frac{2}{5} \sqrt{\Delta\cdot\max\left\{1/L,\sum_{t=1}^{T-1}\eta_t\right\}}\right) \\
    & = \Pr \left(\frac{x_t - \E x_t}{\sqrt{\Var [x_t]}} \geq -\frac{\left( \frac{1}{2} - \frac{2}{5} \right) \sqrt{\Delta\cdot\max\left\{1/L,\sum_{t=1}^{T-1}\eta_t\right\}}}{ \sqrt{\sigma^2(T-1) / (L^2 d)}}\right) \\
    & \geq \Pr \left(\frac{x_t - \E x_t}{\sqrt{\Var [x_t]}} \geq -\frac{\left( \frac{1}{2} - \frac{2}{5} \right) \sqrt{\Delta\cdot\max\left\{1/L,\sum_{t=1}^{T-1}\eta_t\right\}}}{ \sqrt{\sigma^2(T-1) / (L^2 d_0)}}\right) = 1- \delta/T,
\end{align*}
and furthermore, the same bound holds almost surely for all $t$ with $\Var [x_t]=0$.
Finally, taking a union bound over $t$, we conclude
that this lower bound holds for all $x_t$ with probability $1-\delta$, which implies \eqref{eq:toshoww} as required.

\subsection{Proof of Proposition \ref{prop:noise}}

To prove the proposition, we will need the following Lemma, which formalizes the fact that the norm of high-dimensional Gaussian random variables tend to be concentrated around a fixed value:
\begin{lemma}\label{lem:gausscon}
Let $M,\gamma>0$ be fixed. For any $d$, let $\bx_d$ be a random variable normally distributed with $\bx_d \sim \Ncal(\bu, \frac{\gamma}{d}I_d)$, where $\bu$ is some vector in $\reals^d$ with $\|\bu\|^2=M$. Then for any $\epsilon\in (0,1)$,
\[
	\Pr\left(\left|\frac{\norm{\bx_d}^2}{M+\gamma}-1\right|\leq
	\epsilon\right)~\geq~1-4\exp\left(-\frac{d\epsilon^2}{24}\right).
\]
\end{lemma}
\begin{proof}
	Consider $\bx_d$ for some fixed $d$. We can decompose it as 
	$\bu+\sqrt{\frac{\gamma}{d}}\bn$, where $\bn$ has a standard Gaussian distribution in $\reals^d$ (zero mean and covariance matrix 
	being the identity). Thus, 
	\begin{align}
	\frac{\norm{\bx_d}^2}{M+\gamma}-1 ~&=~ 
	\frac{\norm{\bu}^2+2\sqrt{\frac{\gamma}{d}}\bu^\top\bn+\frac{\gamma}{d}\norm{\bn}^2}
	{M+\gamma}-1~=~
	\frac{2\sqrt{\frac{\gamma}{d}}\bu^\top\bn+\gamma\left(\frac{1}{d}\norm{\bn}^2-1\right)}
	{M+\gamma}\notag\\
	&=~\frac{2\sqrt{\gamma/d}}{M+\gamma}\bu^\top\bn
	+\frac{\gamma}{M+\gamma}\left(\frac{1}{d}\norm{\bn}^2-1\right). \label{eq:twoterms}
	\end{align}
	The first term in the sum above is distributed as a Gaussian in $\reals$ with zero mean and variance 
	$\frac{4\gamma}{d(M+\gamma)^2}\norm{\bu}^2=\frac{4\gamma M}{d(M+\gamma)^2}\leq \frac{4\gamma M}{d\cdot 2\gamma M}=\frac{2}{d}$. By a standard Gaussian tail bound, it follows that the probability that it exceeds $\epsilon/2$ in absolute value is at most $2\exp(-d\epsilon^2/16)$. Similarly, for the second term, we have by a standard tail bound for Chi-squared random 
	variables (see for example \cite[Lemma~B.12]{shalev2014understanding}) that
	\begin{align*}
	& \Pr\left(\frac{\gamma}{M+\gamma}\left|\frac{1}{d}\norm{\bn}^2-1\right|\geq \frac{\epsilon}{2}\right) \\
	& \leq \Pr\left(\left|\frac{1}{d}\norm{\bn}^2-1\right|\geq \frac{\epsilon}{2}\right) \leq 2\exp(-d\epsilon^2/24).
	\end{align*}
	Combining the above with a union bound, it follows that \eqref{eq:twoterms} has absolute value more than $\epsilon$ with probability at most 
	\[
	2\exp(-d\epsilon^2/16)+2\exp(-d\epsilon^2/24)
	~\leq~ 4\exp(-d\epsilon^2/24)~.
	\]
\end{proof}

\begin{proof}[Proof of \propref{prop:noise}]
    We will utilize the function
    \[
    	f(\bx) = \frac{L}{2}\norm{\bx}^2,
    \]
    where $\bx_1$ is some vector such that $\norm{\bx_1}=\sqrt{\Delta/L}$. Using a derivation similar to the 
    one used in \eqref{eq:xtexpression}, we have
    \[
        \bx_{t+1}= \bx_t-\eta_t \cdot (L \bx_t + \bxi_t) = (1-L\eta_t) \bx_t - \eta_t \bxi_t,
    \]
    hence
    \begin{equation}\label{eq:xtexpression2}
    \bx_t = \prod_{j=1}^{t-1}\left(1-L\eta_j\right)\bx_1-
    \sum_{j=1}^{t-1}\eta_j 
    \prod_{i=j+1}^{t-1}\left(1-L\eta_i\right)\bxi_j.
    \end{equation}
    Since each $\bxi_j$ is an independent zero-mean Gaussian with covariance matrix 
    $\frac{\sigma^2}{d}I_d$, we get that $\bx_t$ has a Gaussian distribution with mean 
    $\prod_{j=1}^{t-1}\left(1-L\eta_j\right)\bx_1$ and covariance matrix
    $\frac{\gamma_t}{d}I_d$, where
    \[
    \gamma_t = \sigma^2 \sum_{j=1}^{t-1}\eta_j^2
    \prod_{i=j+1}^{t-1}\left(1-L\eta_i\right)^2.
    \]
    By \lemref{lem:gausscon}, taking $\epsilon=1/2$ and $d\geq d_0$ with
    \[
        d_0 := 96 \log \frac{4T}{\delta} = \Ocal(\log(T/\delta)),
    \]
    it follows that $\norm{\nabla f(\bx_t)}^2=\norm{L \bx_t}^2$ is at least
    \begin{equation}\label{eq:lowboundthis}
    \frac{L}{2}\cdot\left( \Delta\prod_{j=1}^{t-1}\left(1 - L\eta_j\right)^2
    + L \sigma^2 \sum_{j=1}^{t-1}\eta_j^2
    \prod_{i=j+1}^{t-1}\left(1-L\eta_i\right)^2\right)
    \end{equation}
    with probability at least $1-\delta/T$.
    Our goal now will be to lower bound~\eqref{eq:lowboundthis} under the conditions in the proposition. Plugging this lower bound and applying a union bound over all $t\in [T]$ will result in our proposition.
    \begin{itemize}
    	\item If $\eta_t=\eta$ and $\eta\in [0,1/L)$, we can lower bound \eqref{eq:lowboundthis} by 
    	\begin{align*}
    	& \frac{L}{2} \left( \Delta(1-L\eta)^{2(t-1)} + L \sigma^2 \sum_{j=1}^{t-1} \eta^2 (1-L\eta)^{2(t-1-j)}
    	\right) \\
    	& = \frac{L}{2} \left(\Delta(1 - L\eta)^{2(t-1)} + L \sigma^2 \eta^2\cdot \frac{1-(1-L\eta)^{2(t-1)}}{1-(1-L\eta)^2} \right) \\
    	&= \frac{L}{2} \left( \Delta(1 - L\eta)^{2(t-1)} + \frac{\eta \sigma^2}{2-L\eta}\left(1-(1-L\eta)^{2(t-1)}\right)\right).
    	\end{align*}
    	For any $t$, this is a convex combination of $\frac{L}{2} \Delta$ and 
    	$\frac{L}{2} \frac{\eta \sigma^2}{2-L\eta}$, hence is at least the minimum between them. 
    	\item If there exists some constant $c\geq 0$ such that $\eta_t \geq c/L$ for all $t$, we can lower bound \eqref{eq:lowboundthis} by $\frac{L^2}{2} \sigma^2 \eta_{t-1}^2\geq \frac{\sigma^2 c^2}{2}$ (i.e., accounting for the noise at the last iterate).
    	\item If $\eta_t=\frac{a}{L(b+t^\theta)}$ (where $a>0, b\geq 0,\theta\in (0,1/2)$), then it is easily verified that for a certain constant $\tau_{a,b,\theta}$ depending only on $a,b,\theta$,
    	\[
    	1\leq\frac{1}{L\eta_t}\leq \frac{t}{2}~~\text{for all}~~t\geq \tau_{a,b,\theta}.
    	\]
    	In that case, we can lower bound \eqref{eq:lowboundthis} by
    	\begin{align*}
    	    & \frac{L^2 \sigma^2}{2} \sum_{j=t-\lfloor 1/(L\eta_t)\rfloor}^{t-1}\eta_j^2\prod_{i=j+1}^{t-1}(1-L\eta_i)^{2} \\
    	    & \geq \frac{L^2 \sigma^2 \eta_t^2}{2}\cdot \left\lfloor\frac{1}{L\eta_t}\right\rfloor\left(1-L\eta_{\lfloor t/2\rfloor}\right)^{2\lfloor 1/(L\eta_t) \rfloor} \\
    		& \geq \frac{\sigma^2 L \eta_t}{4}\left(1-\frac{a}{b+\lfloor t/2\rfloor^{\theta}}\right)^{2\left\lfloor \frac{b+t^\theta}{a}\right\rfloor},
    	\end{align*}
    	which is at least $c_{a,b,\theta} \sigma^2 L \eta_t\geq c_{a,b,\theta} \sigma^2 L \eta_T$ if $t\geq \tau'_{a,b,\theta}$ (for some parameters $c_{a,b,\theta},\tau'_{a,b,\theta}$ depending on $a,b,\theta$). Moreover, if $t<\tau'_{a,b,\theta}$, then \eqref{eq:lowboundthis} is at least
    	\[
    	\frac{L^2 \sigma^2}{2} \eta_{t-1}^2 \geq \frac{L^2 \sigma^2}{2} \eta_{\tau'_{a,b,\theta}}^2 = \frac{\sigma^2}{2} \cdot \left( \frac{a}{b+(\tau'_{a,b,\theta})^\theta}\right)^2.
    	\]
    	Combining both cases, we get that \eqref{eq:lowboundthis} is at least  $c'_{a,b,\theta} \sigma^2 \cdot\min\{1,L \eta_T\}$, where $c'_{a,b,\theta}$ is again some constant dependent on $a,b,\theta$, implying the stated result.
    \end{itemize}
\end{proof}

\section{Upper Bounds for SGD}\label{sec:upperbounds}

In order to place our lower bounds in perspective, we state and prove a rather standard $\Ocal(\epsilon^{-4})$ complexity bound for SGD, which unlike the result discussed in the introduction, does not assume anything special about the Hessians or the noise, and is completely independent of the dimension.

We start the analysis with a technical lemma that we will use to derive bounds both in the stochastic and deterministic settings.
\begin{lemma}\label{L:nonconvex_upper}
Consider the Stochastic Gradient Descent
\[
    \bx_{t+1} = \bx_t - \eta_{t} \left(\nabla f(\bx_t) + \bxi_t\right), \quad t\in [T-1],
\]
where $0 < \eta_t < 1/L$, $f$ is a non-convex function with $L$-Lipschitz gradient, and $\bxi_t$ is a random noise with $\E(\bxi_t)=0$, $V(\bxi_t)=\sigma^2$. Then for any choice of $\kappa_t$, $t\in [T-1]$ such that $1-L\eta_t\leq \kappa_t \leq (1-L\eta_t)^{-1}$ we have
\[
    \min_{t \in [T]} \E \|\nabla f(\bx_t)\|^2 \leq \frac{4L(f(\bx_1)-f(\bx_*)) + \sum_{t=1}^{T-1} \frac{L^2 \eta_{t}^2(1-L \eta_{t}+\kappa_{t})}{1 - L\eta_{t}} \sigma^2}{3 (T-1) - \sum_{t=1}^{T-1} (1 - L \eta_t) (1 - L \eta_t + \kappa_t + \frac{1}{\kappa_{t}})},
\]
where $x_*$ is a stationary point with $f(\bx_*) \leq f(\bx_T)$.
\end{lemma}
\begin{proof}
By \thmref{T:nonconvex_interpolation} we have
\begin{align*}
    & \frac{1}{2L} \|\nabla f(\bx_{t})-\nabla f(\bx_{t+1})\|^2 - \frac{L}{4} \|\bx_{t} - \bx_{t+1} - \frac{1}{L}(\nabla f(\bx_{t})-\nabla f(\bx_{t+1}))\|^2 \\&\quad\overset{\text{a.s.}}{\leq} f(\bx_{t}) - f(\bx_{t+1}) - \langle \nabla f(\bx_{t+1}), \bx_{t}-\bx_{t+1}\rangle, \quad t\in [T-1],
\end{align*}
which by the definition of $\bx_{t+1}$ becomes
\begin{align*}
    & \frac{1}{2L} \|\nabla f(\bx_{t})-\nabla f(\bx_{t+1})\|^2 - \frac{1}{4L} \|L \eta_{t} \bxi_t - (1-L \eta_{t})\nabla f(\bx_{t}) + \nabla f(\bx_{t+1})\|^2 \\&\quad \overset{\text{a.s.}}{\leq} f(\bx_{t}) - f(\bx_{t+1}) - \langle \nabla f(\bx_{t+1}), \eta_{t} \left(\nabla f(\bx_t) + \bxi_t\right)\rangle, \quad t\in [T-1],
\end{align*}
Adding up the inequality above for all $t\in [T-1]$ brings us to
\begin{align*}
    & \frac{1}{2L} \sum_{t=1}^{T-1} \|\nabla f(\bx_{t})-\nabla f(\bx_{t+1})\|^2 - \frac{1}{4L} \sum_{t=1}^{T-1} \|L \eta_{t} \bxi_t - (1-L \eta_{t})\nabla f(\bx_{t}) + \nabla f(\bx_{t+1})\|^2 \\&\quad+ \sum_{t=1}^{T-1} \eta_{t} \langle \nabla f(\bx_{t+1}), \nabla f(\bx_t) + \bxi_t \rangle \overset{\text{a.s.}}{\leq} f(\bx_1) - f(\bx_{T}),
\end{align*}
which, after adding
\[
\frac{1}{4} \sum_{t=1}^{T-1} \eta_{t} (1-L \eta_{t}+\kappa_{t}) \left(\frac{L \eta_{t}}{1-L \eta_{t}} \|\bxi_t\|^2 + 2 \langle \nabla f(\bx_t), \bxi_t\rangle \right)
\]
to both sides and rearranging the terms, brings us to
\begin{align*}
    & \frac{1}{4L} \sum_{t=1}^{T-1} \left( 2 - (1-L \eta_t) (1 - L \eta_t + \kappa_t) \right) \|\nabla f(\bx_t)\|^2 + \frac{1}{4L} \sum_{t=1}^{T-1} \left(1 - \frac{1-L \eta_{t}}{\kappa_{t}}\right) \|\nabla f(\bx_{t+1})\|^2 \\&\quad + \frac{1}{4L} \sum_{t=0}^{T-1}(1 - L \eta_t)\kappa_t \left\| \nabla f(\bx_t) - \frac{1}{\kappa_t} \nabla f(\bx_{t+1}) -  \frac{L \eta_t}{1-L \eta_t}\bxi_t \right\|^2 \\&\quad \overset{\text{a.s.}}{\leq} f(\bx_1) - f(\bx_{T}) + \frac{1}{4} \sum_{t=0}^{T-1} \eta_{t} (1-L \eta_{t}+\kappa_{t}) \left(\frac{L \eta_{t}}{1-L \eta_{t}} \|\bxi_t\|^2 + 2 \langle \nabla f(\bx_t), \bxi_t\rangle \right)
\end{align*}
Finally, taking the expected value of both side, and noting that $\E \|\bxi_t\|=\sigma$, $\E \langle \nabla f(\bx_t), \bxi_t\rangle=0$, and $\E f(\bx_T)\geq f(\bx_*)$, we reach
\begin{align*}
    & \frac{1}{4L} \left(3 (T-1) - \sum_{t=1}^{T-1}  (1-L \eta_t) (1 - L \eta_t + \kappa_t) - \sum_{t=1}^{T-1} \frac{1-L \eta_t}{\kappa_t} \right)\min_{t\in [T]} \E \|\nabla f(\bx_t)\|^2 \\&\quad \leq f(\bx_1) - f(\bx_*) + \frac{1}{4} \sum_{t=1}^{T-1} \eta_{t} (1-L \eta_{t}+\kappa_{t}) \frac{L \eta_{t}}{1-L \eta_{t}} \sigma^2,
\end{align*}
concluding the proof.
\end{proof}

An explicit optimal expression for $\kappa_i$ appears to be complex in the general case, however, for two important cases a good approximation can be obtained. First, when $\sigma$ is large, the term in the numerator dominates the expression, thus the optimal value for $\kappa_t$ approaches $1-L \eta_t$ as $\sigma\rightarrow\infty$, recovering the following result by Ghadimi and Lan:
\begin{theorem}[{\cite[Theorem~2.1]{ghadimi2013stochastic}}]\label{T:upper_noncovex}
Consider the fixed-step Stochastic Gradient Descent
\[
    \bx_{t+1} = \bx_t - \eta_t \left(\nabla f(\bx_t) + \bxi_t\right), \quad t \in [T-1],
\]
where $f$ is a nonconvex function with $L$-Lipschitz gradient, $\bxi_t$ is a random noise with $\E(\bxi_t)=0$, $V(\bxi_t)=\sigma^2$ and $0 < L \eta_t < 1$. Then
\begin{equation}\label{eq:sgp_upper_bound}
    \min_{t \in [T]} \|\nabla f(\bx_t)\|^2 \leq \frac{2(f(\bx_1)-f(\bx_*)) + L \sum_{t=1}^{T-1} \eta_t^2 \sigma^2}{\sum_{t=1}^{T-1} \eta_t (2 - L \eta_t)}.
\end{equation}
where $x_*$ is a stationary point with $f(\bx_*) \leq f(\bx_T)$.
\end{theorem}
\begin{proof}
The result follows directly from \lemref{L:nonconvex_upper}, taking $\kappa_t = 1-L \eta_t$.
\end{proof}

A second case where a simple expression for $\kappa_t$ can be easily attained is when $\sigma=0$, i.e., in the deterministic case. Here an optimal choice for $\kappa$ is $\kappa_i=1$, giving the following result which appears to be a new and slightly improved version of the classical result by Nesterov~\cite{Book:Nesterov}, eq.~(1.2.15):
\begin{corollary}\label{C:gd_upper_bound}
Consider the fixed-step Gradient Descent
\[
    \bx_{t+1} = \bx_t - \eta_t \nabla f(\bx_t), \quad t \in [T-1],
\]
where $f$ is a nonconvex function with $L$-Lipschitz gradient and $0 < \eta_t < 1/L$. Then
\[
    \min_{t \in [T]} \|\nabla f(\bx_t)\|^2 \leq \frac{4 (f(\bx_1)-f(\bx_*))}{\sum_{t=1}^{T-1} \eta_t(4-L \eta_t)},
\]
where $x_*$ is a stationary point with $f(\bx_*) \leq f(\bx_T)$.
\end{corollary}

\begin{remark}
The discovery of the proof of \lemref{L:nonconvex_upper} was guided by numerically solving an optimization problem called the Performance Estimation Problem, whose solution captures the worst-case performance of the SGD method. This technique was first introduced in~\cite{Article:Drori} and was later shown in~\cite{taylor2015smooth} to achieve tight bounds for a wide range of methods in the deterministic case. This, in conjunction with the nearly matching lower bound established in \thmref{thm:aggregation_step}, motivates us to raise the conjecture that \lemref{L:nonconvex_upper} gives a tight bound (including the constant) in the stochastic case.
\end{remark}

\end{document}